\documentclass[11pt,a4paper]{article}

\usepackage{amsmath,amssymb,amsthm,booktabs} % For formal tables
\usepackage{xspace}
\usepackage[algo2e,ruled,vlined,linesnumbered]{algorithm2e}
\usepackage{mathtools}

\renewcommand\labelenumi{(\roman{enumi})}
\renewcommand\theenumi\labelenumi

\newcommand{\assign}{\leftarrow}
\newcommand{\fmutb}{\texttt{fmut}_\beta}
\DeclareMathOperator{\random}{random}
\DeclareMathOperator{\poly}{poly}

\newcommand{\onemax}{\textsc{OneMax}\xspace}
\newcommand{\leadingones}{\textsc{LeadingOnes}\xspace}
\newcommand{\jump}{\textsc{Jump}\xspace}
\newcommand{\N}{\mathbb{N}}
\newcommand{\R}{\mathbb{R}}
\newcommand{\Z}{\mathbb{Z}}
\renewcommand{\epsilon}{\varepsilon}
\newcommand{\eps}{\varepsilon}
\newcommand{\oea}{$(1 + 1)$~EA\xspace}
\newcommand{\foea}{$(1 + 1)$~FEA$_\beta$\xspace}
\newcommand{\lea}{$(1 + \lambda)$~EA\xspace}

\DeclareMathOperator{\opt}{opt}

\newcommand{\expect}[1]{\mathbb{E}\left( #1 \right)}
\newcommand{\card}[1]{\left\vert #1 \right\vert}

\newtheorem{theorem}{Theorem}
\newtheorem{lemma}[theorem]{Lemma}
\newtheorem{corollary}[theorem]{Corollary}

\begin{document}
{\sloppy

\title{Fast Genetic Algorithms}

\author{Benjamin Doerr\thanks{Laboratoire d'Informatique, {\'E}cole Polytechnique, Palaiseau, France,\newline \tt{email: lastname@lix.polytechnique.fr}}
\and Huu Phuoc Le\thanks{{\'E}cole Polytechnique, Palaiseau, France}
\and R\'egis Makhmara\thanks{Laboratoire d'Informatique, {\'E}cole Polytechnique, Palaiseau, France}
\and Ta Duy Nguyen\thanks{{\'E}cole Polytechnique, Palaiseau, France}}

\maketitle

\begin{abstract}

  For genetic algorithms using a bit-string representation of length~$n$, the general recommendation is to take $1/n$ as mutation rate. In this work, we discuss whether this is really justified for multimodal functions. Taking jump functions and the $(1+1)$ evolutionary algorithm as the simplest example, we observe that larger mutation rates give significantly better runtimes. For the $\jump_{m,n}$ function, any mutation rate between $2/n$ and $m / n$ leads to a speed-up at least exponential in $m$ compared to the standard choice. 
  
  The asymptotically best runtime, obtained from using the mutation rate $m/n$ and leading to a speed-up super-exponential in $m$, is very sensitive to small changes of the mutation rate. Any deviation by a small $(1 \pm \eps)$ factor leads to a slow-down exponential in $m$. Consequently, any fixed mutation rate gives strongly sub-optimal results for most jump functions.
   
  Building on this observation, we propose to use a random mutation rate $\alpha/n$, where $\alpha$ is chosen from a power-law distribution. We prove that the $(1+1)$ EA with this heavy-tailed mutation rate optimizes any $\jump_{m,n}$ function in a time that is only a small polynomial (in~$m$) factor above the one stemming from the optimal rate for this $m$. 
  
  Our heavy-tailed mutation operator yields similar speed-ups (over the best known performance guarantees) for the vertex cover problem in bipartite graphs and the matching problem in general graphs. 
  
  Following the example of fast simulated annealing, fast evolution strategies, and fast evolutionary programming, we propose to call genetic algorithms using a heavy-tailed mutation operator \emph{fast genetic algorithms}.
\end{abstract}

%\begin{CCSXML}
%<ccs2012>
%  <concept_id>10010520.10010553.10010562</concept_id>
%  <concept_desc>Computer systems organization~Embedded systems</concept_desc>
%  <concept_significance>500</concept_significance>
%  <concept_id>10010520.10010575.10010755</concept_id>
%  <concept_desc>Computer systems organization~Redundancy</concept_desc>
%  <concept_significance>300</concept_significance>
%  <concept_id>10010520.10010553.10010554</concept_id>
%  <concept_desc>Computer systems organization~Robotics</concept_desc>
%  <concept_significance>100</concept_significance>
%  <concept_id>10003033.10003083.10003095</concept_id>
%  <concept_desc>Networks~Network reliability</concept_desc>
%  <concept_significance>100</concept_significance>
%</ccs2012>  
%\end{CCSXML}
%
%\ccsdesc[500]{Computer systems organization~Embedded systems}
%% We no longer use \terms command
%%\terms{Theory}

%\keywords{Evolutionary algorithm, mutation operator, heavy-tailed distribution, power-law distribution, multimodal optimization.}

%\onecolumn

\section{Introduction}

One of the basic variation operators in evolutionary algorithmics is mutation, which is generally understood as a mild modification of a single parent individual. When using a bit-string representation, the most common mutation operator is \emph{standard-bit mutation}, which flips each bit of the parent bit-string $x \in \{0,1\}^n$ independently with some probability $p_n$. The general recommendation is to use a \emph{mutation rate} of $p_n = 1/n$. The expected number of bits parent and offspring differ in then is one. $p_n = 1/n$ is also the mutation rate which maximizes the probability to create as offspring a Hamming neighbor $y$ of the parent $x$, that is, $y$ differs from $x$ in exactly one bit. This mutation rate also gives the asymptotically optimal expected optimization times for several simple evolutionary algorithms on classic simple test problems (see Subsection~\ref{subsec:related} for the details).

In this work, we argue that the $1/n$ recommendation could be the result of an over-fitting to these simple unimodal test problems. As a first indication for this, we determine the optimal mutation rate for optimizing \emph{jump functions}, which were introduced in~\cite{DrosteJW02}. The function $\jump_{m,n}$, $m \ge 2$, differs from the simple unimodal OneMax function (counting the number of ones in the bit-string) in that the fitness on last $m-1$ suboptimal fitness levels is replaced by a very small value. Consequently, an elitist algorithm quickly finds a search point on the thin plateau of local optima, but then needs to flip $m$ bits to jump over the fitness valley to the global optimum. 

Denote by $T_p(m,n)$ the expected \emph{optimization time} (number of search points evaluated until the optimum is found) of the $(1+1)$ evolutionary algorithm (EA) with mutation rate $p$ on the function $\jump_{m,n}$. Extending the result of~\cite{DrosteJW02} to arbitrary mutation rates, we observe that for all $m = o(n)$, the classic choice of the mutation rate gives an expected optimization time of \[T_{1/n}(m,n) = (1+o(1)) e n^{m},\] where as the choice of $p_n = m/n$ leads to an expected optimization time of \[T_{m/n}(m,n) =  \big((1+o(1))\tfrac em\big)^m n^{m},\] an improvement super-exponential in $m$. This is optimal apart from lower order terms, that is, $T_{\opt}(m,n) := \min\{T_p(m,n) \mid p \in [0,1/2]\}$ satisfies $T_{\opt} = (1+o(1)) T_{m/n}$.%\merk{to be able to write $(1+o(1))$, we need similar statements in the lemmas} \merk{true for $m=o(n)$ checked! you can write some comment about this under the lemma as a trivial consequence}

This large runtime improvement by choosing an uncommonly large mutation rate may be surprising, but as our proofs reveal there is a good reason for it. It is true that raising the mutation rate from $1/n$ to $m/n$ decreases the rate of $1$-bit flips from roughly $1/e$ to roughly $m e^{-m}$. However, finding a particular Hamming neighbor is much easier than finding the required distance-$m$ search point. Consequently, the factor $me^{-m}$ slow-down of the roughly $n/2$ one-bit improvements occurring in a typical optimization process is significantly outnumbered by the factor $m^me^{-m}$ speep-up of finding the $m$-bit jump to the global optimum. 
%
%The intuitive reason for this discrepancy is that while only once in the optimization process a large jump to a distance-$m$ search point is needed, finding this jump requires flipping the right $m$ bits (out of a total of $n$ bit positions). This takes many attempts regardless of which mutation rate is used. For the mutation rate $1/n$, this is hindered additionally by the fact that only with probability $\binom{n}{m} (1/n)^m (1-1/n)^{n-m} = \Theta((m/e)^{-m-0.5})$ exactly $m$ bits are flipped (where is the mutation rate $m/n$ with probability $\Theta(m^{-0.5})$ flips exactly $m$ bits). 

These observations suggest that the traditional choice of the mutation rate, leading to a maximal rate of $1$-bit flips, is not ideal. Instead, one should rather optimize the mutation rate with the aim of maximizing the rate of the largest required long-distance jump in the search space.  

Continuing with the example of the jump functions, however, we also observe that small deviations from the optimal rate lead to significant performance losses. When optimizing the function $\jump_{m,n}$ with a mutation rate that differs from $m/n$ by a small constant factor in either direction, the expected optimization becomes larger than $T_{m/n}(m,n) \approx T_{\opt}(m,n)$ by a factor exponential in~$m$. Consequently, there is no good one-size-fits-all mutation rate and finding a good mutation rate for an unknown multimodal problem requires a deep understanding of the fitness landscape.

Based on these insights, we propose to use standard-bit mutation not with a fixed rate, but with a rate chosen randomly according to a \emph{heavy-tailed} distribution. Such a distribution ensures that the number of bits flipped is not strongly concentrated around its mean, which eases having jumps of all sizes in the search space. More precisely, the heavy-tailed mutation operator we propose first chooses a number $\alpha \in [1..n/2]$ according to a power-law distribution $D_n^\beta$ with (negative) exponent $\beta > 1$ and then creates the offspring via standard-bit mutation with rate $\alpha/n$. 

This mutation operator shares many desirable properties with the classic operator. For example, the probability that a single bit (or any other constant number of bits) is flipped is constant. This implies that many classic runtime results hold for our new mutation operator as well. Also, any search point can be created from any parent with positive probability. This probability, however, in the worst case is much higher than when using the classic mutation operator. Consequently, the (tight) general $O(n^n)$ runtime bound for the \oea optimizing any pseudo-Boolean function with unique optimum~\cite{DrosteJW02} improves to $O(n^\beta 2^n)$.

For our main example for a multimodal landscape, the jump functions, we prove that the \oea with our heavy-tailed mutation operator finds the optimum of \emph{any} function $\jump_{m,n}$ with $m > \beta-1$ in expected time \[T_{D_n^\beta}(m,n) \le O(m^{\beta-0.5} \big((1+o(1))\tfrac em\big)^m n^m),\] which is again an improvement super-exponential in $m$ over the classic runtime $T_{1/n}(m, n)$, and only a small polynomial factor of $O(m^{\beta-0.5})$ slower than $T_{\opt}(m,n)$, the expected runtime stemming from the mutation rate which is optimal for this $m$. Note that in return for this small polynomial factor loss over the best instance specific mutation rate we obtained a single mutation operator that achieves a near-optimal (apart from this small polynomial factor) performance on all instances. Note further that the restriction $m > \beta-1$ is automatically fulfilled when using a $\beta < 3$, which is both a good choice from the view-point of heavy-tailed distributions and in the light of the $O(m^{\beta-0.5})$ slow-down factor. 

We observe that a small polynomial factor slow-down cannot be avoided when aiming at a competitive performance on all instances. We prove a lower bound result showing that no randomized choice $D$ of the mutation rate can give a performance of $T_D(m,n) = O(m^{0.5} T_{\opt}(m,n))$ for all $m$. Consequently, by choosing $\beta$ close to $1$ we get the essentially the theoretically best performance on all jump functions. 

Some elementary experiments show that the above runtime improvements are visible already from small problem sizes on. For $m = 8$, the \oea using the heavy-tailed operator with $\beta=1.5$ was faster than classic choice by a factor of at least $2000$ on each instance size $n \in \{20, 30, \dots, 150\}$. 
%For $m = 5$, the \oea using the heavy-tailed operator with $\beta=1.5$ was faster than classic choice by a factor of at least $5$ on each instance size $n \in \{20, 30, \dots, 150\}$. For $m=8$, the advantage in the same set of experiments was always by at least a factor of $2000$.

The very precise results above are made possible by regarding the clean test example of the jump functions. However, we observe a similar behavior for two combinatorial optimization problems regarded in the evolutionary computation literature before. One is the problem of computing a \emph{minimum vertex cover} in complete bipartite graphs. If the partition classes have sizes $m$ and $n-m$, then the \oea with mutation rate $1/n$ has an expected optimization time of at least $\Omega(m n^{m-1})$~\cite{FriedrichHHNW10}. With our heavy-tailed mutation operator, the optimization time drops to $O(n^{\beta} 2^m)$ for all instance with $m \le n/3$. Note that for larger values, our general bound of $O(n^{\beta} 2^n)$ also gives a significant improvement over the classic result, though this is maybe less interesting as a performance of $O(2^n)$ could be also obtained with random search.

As a second combinatorial optimization problem, we regard the problem of computing a large \emph{matching} in an arbitrary undirected graph. For this problem, it was shown in~\cite{GielW03,GielW04} that the \oea with mutation rate $1/n$ finds a matching $M$ with cardinality $|M| \ge OPT / (1+\eps)$ in time $O(n^{2 \lceil 1/\eps \rceil})$. When using our heavy-tailed mutation operator, this bound improves to $O\big(\big((1+o(1))\tfrac em\big)^m m^{\beta-0.5} n^{m+1}\big)$, where we used shorthand $m := 2 \lceil 1/\eps \rceil - 1$ and the constants implicit in the asymptotic notation are independent of $m$.

Overall, these results indicate that multimodal optimization problems might need mutation operators that move faster through the search space than standard-bit mutation with mutation rate $1/n$. A simple way of achieving this goal that in addition works uniformly well over all required jump sizes is the heavy-tailed mutation operators suggested in this work. To the best of our knowledge, this is the first time that a heavy-tailed mutation operator is proposed for discrete evolutionary algorithms. Heavy-tailed mutation operators have been regarded before in simulated annealing~\cite{SzuH87,SzuH87procieee}, evolutionary programming~\cite{YaoLL99}, evolution strategies~\cite{YaoL97} and other subfields of evolutionary computation, however, always in continuous search spaces. Since these algorithms were called \emph{fast} by their inventors, that is, fast simulated annealing, fast evolutionary programming, and fast evolution strategies, for reasons of consistency we shall call genetic algorithms employing such operators \emph{fast genetic algorithms}, well aware of the fact that this first scientific work regarding heavy-tailed mutation in discrete search spaces does by far not give a complete picture on this approach. The results obtained in this work, however, indicate that this is a promising direction deserving more research efforts.

\section{Related Work}\label{subsec:related}

\subsection{Static Mutation Rates}

For reasons of space, we cannot discuss the whole literature on what is the right way to choose the \emph{mutation rate}, that is, the expected fraction of the bit positions in which parent and mutation offspring differ. Restricting ourselves to evolutionary algorithms for discrete optimization problems, the long-standing recommendation, based, e.g., on~\cite{Back93,Muhlenbein92} is that a mutation rate of $1/n$, that is, flipping in average one bit, is a good choice. A mutation rate of roughly this order of magnitude is used in many experimental works. Nevertheless, in particular in evolutionary algorithms using crossover, the interplay between mutation and crossover may ask for a different choice of the mutation rate. For example, in algorithms using first crossover and then applying mutation to the crossover offspring, a smaller mutation rate can be used to implicitly reduce the mutation probability, that is, the probability that an individual is subject to mutation at all. The $(1+(\lambda,\lambda))$ GA~\cite{DoerrDE15} works best with a higher mutation rate, because it uses crossover with the parent as repair mechanism after the mutation phase. 

For simple mutation-based algorithms, which are the best object to study the working principles of mutation in isolation, the following results have been proven. For the \oea, it was shown that $p=1/n$ is asymptotically the unique best mutation rate for the class of all pseudo-Boolean linear functions~\cite{Witt13}. For the \leadingones test function, a slightly higher rate of approximately $1.59/n$ is optimal~\cite{BottcherDN10}. For the \lea optimizing \onemax, a mutation rate of $1/n$ is again optimal, though for larger value of $\lambda$ any $\Theta(1/n)$ mutation rate gives an asymptotically optimal runtime~\cite{GiessenW15}. 

\subsection{Dynamic Mutation Rates}

Since our heavy-tailed mutation operator can be seen as a dynamic choice of the mutation rate (according to a relatively trivial dynamics), let us quickly review the few results close to ours. There is a general belief that a dynamic choice of the mutation rate can be profitable, typically starting with a higher rate and reducing it during the run of the algorithm. Despite this, dynamic choices of the mutation rate are still not that often seen in today's applied research. On the theory side, the first work~\cite{JansenW05} analyzing a dynamic choice of the mutation strength proposes to take in iteration $t$ the mutation rate $2^{(t-1) \!\! \mod (\lceil \log_2 n \rceil-1)} /n$. In other words, the mutation rates $1/n, 2/n, 4/n, ..., 2^{\lceil \log_2 n \rceil-2}/n$ are used in a cyclic manner. The \oea using this dynamic mutation rate has an expected runtime larger by a factor of $\Theta(\log n)$ for several classic test problems. On the other hand, there are problems where this dynamic EA has a polynomial runtime, whereas all static choices of the mutation rate lead to an exponential runtime. %, and another example function displaying the opposite behavior. 
We remark without proof that these results would also hold if the mutation rate was chosen in each iteration uniformly at random from the set of these powers of two. We note without formal proof that the arguments used in the proof of Theorem~\ref{thm:optp} together with Corollary~\ref{cor:optp} show that either version of this dynamic EA would have a runtime of $\exp(\Omega(m)) T_{\opt}(m,n)$ on $\jump_{m,n}$ for most values of $m$ (namely all that are a small constant factor away from the nearest power of two) and all values of $n$.

For the classic test functions, the following is known. In~\cite{BottcherDN10}, it was shown that the optimal fitness-dependent choice of the mutation rate for the \leadingones test function is $p(x)= \frac{1}{\leadingones(x)+1}$ when the parent is $x$. For the $\oea$, this gives an expected optimization time (apart from lower order terms) of $0.68n^2$ compared to $0.77n^2$ for the optimal static mutation rate and $0.86n^2$ for the static choice $1/n$. For the optimization of \onemax using the \oea, a dynamic mutation rate is known to give runtime improvements only of lower order. Surprisingly, for the \lea a dynamic choice of the mutation rate can lead to an asymptotically better runtime~\cite{BadkobehLS14}. The optimization time of $O(\frac{n \lambda \log\log \lambda}{\log \lambda} + n \log n)$ when using the static mutation rate of $1/n$ improves to $O(\frac{n \lambda}{\log \lambda} + n \log n)$ when using the dynamic choice $p(x) = \max\{\frac{\ln(\lambda)}{n \ln(en/(n-\onemax(x)))},\frac 1n\}$. We note without formal proof that for jump functions, a fitness-dependent mutation rate cannot give a significant improvement over the best static mutation as can be seen from our analysis in Section~\ref{sec:static}.

\subsection{Heavy-Tailed Mutation Operators}

The idea to use heavy-tailed mutation operators is not new in evolutionary computation, and more generally, heuristic optimization. However, it was so far restricted to continuous optimization. Szu and Hartley~\cite{SzuH87,SzuH87procieee} suggested to use a (heavy-tailed) Cauchy distribution instead of Gaussian distributions in simulated annealing and report significant speed-ups. This idea was taken up in evolutionary programming~\cite{YaoLL99}, %IEEE TEC
in evolution strategies~\cite{YaoL97}, estimation of distribution algorithms (EDA)~\cite{Posik09}, %Petr Posik: BBOB-benchmarking a simple estimation of distribution algorithm with cauchy distribution. GECCO (Companion) 2009: 2309-2314
and in natural evolution strategies~\cite{SchaulGS11}. However, also some doubts on the general usefulness of heavy-tailed mutations have been raised. Based on mathematical considerations and experiments, it has been suggested that heavy-tailed mutations are useful only if the large variations of these operators take place in a low-dimensional subspace and this space contains the good solutions of the problem~\cite{HansenGAK06}. Otherwise, the curse of dimensionality makes it just too improbable that a long-range mutation finds a better solution. Also,~\cite{Rudolph97} has pointed out that spherical Cauchy distributions lead to the same order of local convergences as Gaussian distributions, whereas non-spherical Cauchy distributions even lead to a slower local convergence. A heavy-tailed mutation EDA was shown to be significantly inferior to BIPOP-CMA-ES via the BBOB algorithm comparison tool~\cite{Posik10}. %Petr Posík: Comparison of cauchy EDA and BIPOP-CMA-ES algorithms on the BBOB noiseless testbed. GECCO (Companion) 2010: 1697-1702

%Algorithms using heavy-tailed distributions are usually called fast, i.e., fast simulated annealing~\cite{SzuH87}, fast evolution strategies, and fast genetic programming~\cite{YaoLL99}. Consequently, for reasons of consistency, we propose to call genetic algorithms, or more broadly evolutionary algorithms, using heavy-tailed offspring distributions \emph{fast genetic algorithms} or \emph{fast evolutionary algorithms} despite the fact that one scientific work is not enough to evaluation to what extent these methods are generally faster than the classic approach of using independent bit-flips.

\section{Preliminaries}

Throughout this paper, we use the following elementary notation. For $a,b \in \R$, we write $[a..b] := \{z \in \Z \mid a \le z \le b\}$ to denote the set of integers in the real interval $[a,b]$. We denote by $\N$ the set of \emph{positive} integers and by $\N_0$ the set of non-negative integers. For $n,m \in \N_0$ with $m \le n$, we write $\binom{n}{\le m} := \sum_{i=0}^m \binom{n}{i}$ for the number of subsets of an $n$-element set that have at most $m$ elements. For two bit-strings $x, y \in \{0,1\}^n$ of length $n$, we denote by $H(x,y) := \{i \in [1..n] \mid x_i \neq y_i\}$ the \emph{Hamming distance} of $x$ and $y$.

\subsection{Jump Functions}

In this work, we investigate the influence of the mutation operator on the performance of genetic algorithms optimizing multimodal functions. We restrict ourselves to \emph{pseudo-Boolean} functions, that is, functions $f : \{0,1\}^n \to \R$ defined on bit-strings of a given length $n$. As much as the \onemax test function $\onemax_n : \{0,1\}^n \to \R; x \mapsto \sum_{i=1}^n x_i$ is the prime example to study the optimization on easy unimodal fitness landscapes, the most popular test problem for multimodal landscapes are jump functions. For $n \in \N$ and $m \in [1..n]$, Droste, Jansen, and Wegener~\cite{DrosteJW02} define the $n$-dimensional jump function $\jump_{m,n} : \{0,1\}^n \to \R$ by
\[
\jump_{m,n}(x) = \begin{cases}
m + \onemax_{n}(x) &\mbox{ if } \onemax_{n}(x) \le n - m \\ 
				&\;\;\;\;\;\mbox{ or } \onemax_n(x) = n\\
n - \onemax_n(x) &\mbox{ otherwise}
\end{cases}
\]
for all $x \in \{0,1\}^n$. In this paper, we are only interested in the case $m \in [2..n/2]$ when $\jump_{m,n}$ does neither degenerate into $\onemax_n$ nor the local optimum encompasses half the search space or more.
\begin{figure}[h]
\center
\includegraphics[width=0.7\textwidth]{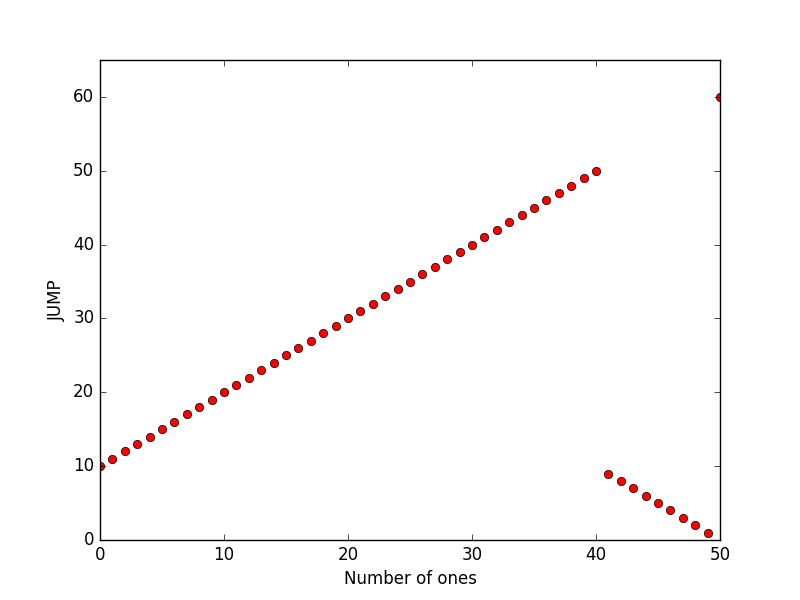}
\caption{The function $\jump_{m,n}$ for $n = 50$ and $m=10$.}
\label{jump_function}
\end{figure}

Jump functions are a useful object to study how well evolutionary algorithms can leave local optima. With the whole radius-$m$ Hamming sphere around the global optimum forming an inferior local optimum of $\jump_{m,n}$, it is very hard for an evolutionary algorithm to not get trapped in this local optimum for a while. Due to the symmetry of the landscape, the only way to leave the local optimum to a better solution is to flip exactly the right $m$ bits. This symmetric and well-understood structure with exactly one fitness valley to be crossed in a typical optimization process makes the jump functions a popular object to study how evolutionary algorithms can cope with local optima. 

Droste et al.~\cite{DrosteJW02} show that the \oea (made precise in the following subsection) with mutation rate $p=1/n$ takes an expected number of $T_{1/n}(m,n) = \Theta(n^m + n \log n)$ fitness evaluations to find the maximum of $\jump_{m,n}$. Here and in the following, all asymptotic notation is to be understood that the implicit constants can be chosen independent of $n$ and $m$. For a broad class of non-elitist algorithms using a mutation rate of $c/n$, an upper bound of $O(n \lambda \log \lambda + (n/c)^m)$ was shown in~\cite{DangL16} for the optimization time on $\jump_{m,n}$. Here $c$ is supposed to be a constant. We are not aware of other runtime analyses for mutation-based algorithms optimizing jump functions.

The jump functions family has also been an example to study in a rigorous manner the effectiveness of crossover. The first work in this direction~\cite{JansenW05} %real royal road...
shows that a simple $(\mu+1)$ genetic algorithm with appropriate parameter settings can obtain a better runtime than mutation-based algorithms. Very roughly speaking, for $m = O(\log n)$ this GA has a runtime of $O(4^m \poly(n))$, reducing the runtime dependence on $m$ from $\Theta(n^m)$ to single-exponential. While this result was the first mathematically supported indication that crossover can be useful in discrete evolutionary optimization, it has, as the authors point out, the limitation that it applies only to a GA that uses crossover very sparingly, namely with probability at most $O(1/mn)$, which is very different from the typical application of crossover. This dependence was mildly improved to $O(m/n)$ along with allowing wider ranges for other parameters in~\cite{KotzingST11}. Interestingly, the research activity on the problem of rigorously proving the usefulness of crossover shifted away from jump functions to real royal road functions~\cite{StorchW04,JansenW05} (which are still similar to jump functions), simplified Ising models~\cite{FischerW05,Sudholt05}, and the all-pairs shortest path problem~\cite{DoerrHK12,DoerrJKNT13}. Only last year, Dang et al.~\cite{Dang1,Dang2} by carefully analyzing the population dynamics could show that a simple GA employing crossover and using natural parameter settings can obtain an expected optimization time of $O(n^{m-1} \log n)$ on $\jump_{m,n}$ for constant $m \ge 3$, which is an $O(n/\log n)$ factor speed-up over comparably simple mutation-based EAs.

\subsection{The \oea}

To study the working principles of different mutation operators, we regard the most simple evolutionary algorithm, the $(1+1)$ evolutionary algorithm (EA). This is a common approach in the theory of evolutionary algorithms, which is based on the experience that results for this simple EA often are valid in a similar manner for more complicated EAs. Without proof, remark that most of our findings in an analogous manner hold for many elitist mutation-based $(\mu+\lambda)$ EAs. 

The \oea, given as Algorithm~\ref{alg:oea}, starts with a random search point $x \in \{0,1\}^n$. In the main optimization loop, it creates a mutation offspring from the parent $x$, which replaces the parent unless is has an inferior fitness. Since our focus is on how long this EA takes to create an optimal solution, we do not specify a termination criterion. 

\begin{algorithm2e}%
	\textbf{Initialization:} Sample $x \in \{0,1\}^n$ uniformly at random\;
 \textbf{Optimization:}
\For{$t=1,2,3,\ldots$}{
Sample $y \in\{0,1\}^n$ by flipping each bit in~$x$ with probability~$p$\,; //mutation step\label{line:mutEA}\\
\lIf{$f(y)\geq f(x)$}{$x \assign y$\,; //selection step}
}
\caption{The (1+1)~EA with static mutation rate $p$ for maximizing~$f\colon\{0,1\}^n\to\R$.}
\label{alg:oea}
\end{algorithm2e}

As usual in theoretically oriented works in evolutionary computation, as performance measure of an evolutionary algorithm we regard the number of fitness evaluations it took to achieve the desired goal. For this reason, we define $T_p(m,n)$ to be the expected number of fitness evaluations the \oea performs when optimizing the $\jump_{m,n}$ function until it first evaluates the optimal solution.  

Since using a mutation rate of more than $1/2$ is not very natural (it means creating an offspring that in average is further away from the parent than the average search point), we shall always assume that our mutation rate $p$ is in $[0,1/2]$. When $p$ depends on the bit-string length $n$, as, e.g., in the recommended choice $p=1/n$, we shall for the ease of reading usually make this functional dependence not explicit (e.g., by writing $p(n)$), but simply continue to write $p$. 
%
%We shall usually not try to give a precise expression for the runtime, but only an asymptotic one. If not explicitly mentioned otherwise, this is to be understood that the hidden constants are independent of $n$, $m$, and $p$. Hence $T_p{m,n} = 

\section{Static Mutation Rates}\label{sec:classic}\label{sec:static}

In this section, we analyze the performance of the \oea on jump functions when employing the standard-bit mutation operator that flips each bit independently with fixed probability $p \in (0,1/2]$. Our main result is that the mutation rate giving the asymptotically best runtime on $\jump_{m,n}$ functions is $p=m/n$, which is far from the standard choice of $1/n$ when $m$ is large. Moreover, we observe that a small constant factor deviation from the $m/n$ mutation rate immediately incurs a runtime increase by a factor exponential in $\Omega(m)$. 

To obtain these results, we first determine (with sufficient precision) the optimization time of the \oea on $\jump_{m,n}$ functions.
%\merk{add version precise apart from lower order terms?}

\begin{theorem}\label{thm:optp}
  For all $n \in \N$, $m \in [2..n/2]$ and $p \in (0,1/2]$, the expected optimization time $T_{p}(m,n)$ of the \oea with mutation rate $p$ on the $n$-dimensional test problem $\jump_{m,n}$ satisfies
\begin{equation*}
 \left(1 - \tbinom{n}{\le m-1} 2^{-n}\right) \frac{1}{p^{m} (1-p)^{n-m}} \le T_{p}(m,n) \le \frac{1}{p^{m} (1-p)^{n-m}} + \frac{2\ln(n/m)}{p(1-p)^{n-1}}.
\end{equation*}
  In particular, if $p \le \frac{m}{n}$, then \[\frac 12 \frac{1}{p^{m} (1-p)^{n-m}} \le T_{p}(m,n) \le 3 \frac{1}{p^{m} (1-p)^{n-m}}.\] Consequently, for any $p \in [\frac 2n, \frac mn]$, $T_p(m,n) \le O( 2^{-m} T_{1/n}(m,n))$.
\end{theorem}

\begin{proof}
We partition the search space into the nonempty sets  $A_i = \{x \in \{0,1\}^n \mid  \lVert x\rVert_1 = i\}$, $0 \leq i \leq n$, which we call levels. We first show the more interesting lower bound. 

Denote by $E_i$ the expected number of iterations it takes to find the optimum when the initial search point is in level $A_{i}$. Denote by $\alpha_{i,j}$ the probability that an iteration starting with a search point in level $A_i$ ends (after mutation and selection) with a search point in level $A_j$. 

First, we prove by induction that for all $0 \leq i \leq n - m $, we have $E_i \geq E_{n - m}$. This inequality trivially holds for $i = n - m$. Suppose that it holds for all $i \in [k + 1..n - m]$. We prove that it also holds for $i = k$. We have 
\[E_k = 1 + \sum_{j = k}^{n - m} \alpha_{k,j}E_j.\]
By our induction hypothesis,
\[(1 - \alpha_{k,k})E_k \geq 1 + \bigg(\sum_{j = k + 1}^{n - m} \alpha_{k,j}\bigg) E_{n - m},\]
and thus
\begin{equation}\label{eq:1}
E_k \geq \frac{1}{1 - \alpha_{k,k}} + \frac{1 - \alpha_{k,k} - \alpha_{k, n}}{1 - \alpha_{k,k}} E_{n - m}.
\end{equation}
Note that $E_{n-m} = \left(p^m(1-p)^{n-m}\right)^{-1}$ and $\alpha_{k, n} = p^{n-k}(1-p)^k$. Hence with $p \leq \frac{1}{2}$ and $k \leq n - m$, we have
\[E_{n-m}\alpha_{k, n} = \left(\frac{p}{1-p}\right)^{n - m - k} \leq 1.\]
Together with (\ref{eq:1}), we obtain $E_k \geq E_{n-m}$. By induction, we conclude that $E_i \geq E_{n - m}$ for all $0 \leq i \leq n - m $. 

Let $x$ denote the random initial search point. Then the above estimate yields
\begin{equation*}
	\begin{split}
	 T_{p}(m,n)  &\geq \sum_{i = 0}^{n - m} \Pr[x \in A_i] E_i \ge \left(1 - \tbinom{n}{\le m-1} 2^{-n}\right)E_{n-m} \\ 
	 & = \left(1 - \tbinom{n}{\le m-1} 2^{-n}\right) \frac{1}{p^{m} (1-p)^{n-m}}.
	\end{split}
\end{equation*}

To prove the upper bound, we use the fitness level method~\cite{Wegener01}. Note that the $A_i$ are fitness levels of $\jump_{m,n}$, however, the  order of increasing fitness is $A_{n-1}$, $A_{n-2}$,..., $A_{n-m+1}$, $A_0$, $A_1$, $A_2$,..., $A_{n-m}$, $A_n$. For $i\in [0..n-1]$, let
\[
s_i \coloneqq
\begin{cases}
\binom {n - i}{1}p(1-p)^{n-1} & \mbox{if } 0 \leq i \leq n - m - 1,\\
p^{m}(1-p)^{n-m} & \mbox{if } i = n - m,\\
\binom {i}{1}p(1-p)^{n-1} & \mbox{if } n - m + 1 \leq i \leq n - 1.
\end{cases}
\]
Then $s_i$ is a lower bound for the probability that an iteration starting with a search point $x\in A_i$ ends with a search point of higher fitness. Thus the fitness level theorem gives the following upper bound for $T_{p}(m,n)$.
\begin{equation*}
\begin{split}
T_{p}(m,n) & \leq \sum_{i = 0}^{n - 1}\frac{1}{s_i} \\
&= \sum_{i = 0}^{n - m-1}\frac{1}{(n - i)p(1-p)^{n-1}} + \frac{1}{p^{m}(1-p)^{n-m}} +\sum_{i = n - m + 1}^{n - 1}\frac{1}{ip(1-p)^{n-1}} \\
% & =\frac{1}{p^{m}(1-p)^{n-m}} +\frac{1}{p(1-p)^{n-1}}\left(\sum_{i = m+1}^{n}\frac{1}{i} + \sum_{i = n - m + 1}^{n - 1}\frac{1}{i}\right)\\
% & \le \frac{1}{p^{m}(1-p)^{n-m}} +\frac{2}{p(1-p)^{n-1}}\sum_{i = m+1}^{n}\frac{1}{i}\\
 & \le \frac{1}{p^{m}(1-p)^{n-m}} +\frac{2\ln(n/m)}{p(1-p)^{n-1}},
\end{split}
\end{equation*}
where we used the estimate \[\sum_{i = m+1}^{n}\tfrac{1}{i} + \sum_{i = n - m + 1}^{n - 1}\tfrac{1}{i} \le 2 \sum_{i = m+1}^{n}\tfrac{1}{i} \le 2 \int_m^n \tfrac 1x dx = 2(\ln n - \ln m) = 2 \ln \tfrac nm.\] 

We now show the second, rougher estimate in our claim. Since $m \le \frac{n}{2}$, by the symmetry of the binomial distribution, we have $\big(1 - \binom{n}{\le m-1} 2^{-n}\big) \ge \frac{1}{2}$, and thus $T_{p}(m,n) \geq \frac{1}{2p^{m} (1-p)^{n-m}}$. If $p \le \frac{m}{n}$, then
\begin{equation*}
\begin{split}
 \frac{2\ln(n/m)}{p(1-p)^{n-1}}&\left(\frac{1}{p^{m}(1-p)^{n-m}}\right)^{-1} = 2\ln\left(\frac{n}{m}\right)\left(\frac{p}{1-p}\right)^{m - 1} \\
 &\le 2\ln\left(\frac{n}{m}\right)\left(\frac{m}{n-m}\right)^{m - 1} \le 2\ln\left(\frac{n}{m}\right)\left(\frac{m}{n-m}\right) \le 2,
\end{split}
\end{equation*}
since $\frac{n}{m}-1 \ge \ln\left(\frac{n}{m}\right)$. Therefore, $T_{p}(m,n) \leq \frac{3}{p^{m} (1-p)^{n-m}}$.

From this, we easily compute the exponential runtime gain claimed in the last sentence of the theorem. Let $p \in [\frac 2n, \frac mn]$. Using that $x \mapsto x^m (1-x)^{n-m}$ is increasing in $[0,m/n]$, we compute \[\frac{T_p(m,n)}{T_{1/n}(m,n)} \le \frac{T_{2/n}(m,n)}{T_{1/n}(m,n)} \le 6 e^2 2^{-m}.\]
\end{proof}

%
%One can prove the lower bound by first proving (by induction for example) that the expected optimization time of the algorithm is higher than the expected time to perform the last $m$-sized jump, (\emph{ie} ${p^{-m}(1 - p)^{m - n}}$). The upper bound can be obtained with a fitness-level argument~\cite{Wegener01}.

From the precise runtime analysis above, we estimate the runtime stemming from the optimal choice for the mutation rate and argue that it can be obtained from using the mutation rate $p = m/n$, but not from too many other mutation rates.
%. We then argue that the same bounds hold for the particular mutation rate $p=m/n$, whereas a small constant factor deviation to a $p$ lying outside the interval $[(1-\eps)m/n, (1+\eps)m/n]$ immediately incurs a runtime increase by a factor exponential in $m\varepsilon^2$. To be more precise, we state in Corollary \ref{cor:optp} that every mutation rate outside the interval above should have the expected runtime greater than $\exp(m\varepsilon^2)T_{\opt}(m,n)$ differ by a constant independent of $m$ and $n$.

\begin{corollary}\label{cor:optp}
The best possible optimization time \[T_{\opt}(m,n) \coloneqq \inf\{T_{p}(m,n) \mid p \in [0,1/2]\}\] for a static mutation rate satisfies \[\frac{1}{2} \frac{n^m}{m^m} \bigg(\frac{n}{n-m}\bigg)^{n-m} \le T_{\opt}(m,n) \le 3\frac{n^m}{m^m} \bigg(\frac{n}{n-m}\bigg)^{n-m} .\]
These bounds also hold for $T_{m/n}(m,n)$, whereas for all $0<\varepsilon < 1$, any mutation rate $p \in [0,1/2] \setminus  [(1-\varepsilon)m/n,(1+\varepsilon)m/n]$ gives a runtime slower than $T_{\opt}(m,n)$ by a factor of at least $\frac 16 \exp(m\eps^2 / 5)$.
\end{corollary}

\begin{proof}
For the upper bound, using Theorem~\ref{thm:optp} we simply estimate
\[T_{\opt}(m,n) \le T_{m/n}(m,n) \le 3\frac{n^m}{m^m} \bigg(\frac{n}{n-m}\bigg)^{n-m}. \]
For the lower bound, elementary calculus shows that $\frac{m}{n}$ is the unique maximum point of $x\mapsto x^m(1-x)^{n-m}$ in the interval $[0,1]$. Therefore, by Theorem~\ref{thm:optp},
\[\frac{1}{2} \frac{n^m}{m^m} \bigg(\frac{n}{n-m}\bigg)^{n-m} \le T_{\opt}(m,n).\]
For $\frac 12 \ge p > (1+\varepsilon)\frac{m}{n}$, we have
\[T_{p}(m,n) \ge \frac{1}{2p^m(1-p)^{n-m}}\ge \frac{1}{2\left(\frac{(1+\varepsilon)m}{n}\right)^m \left(1-\frac{(1+\varepsilon)m}{n}\right)^{n-m}}.\]
Hence, since $e^x \ge 1 +x + \frac{x^2}{2}$ and $1-x\le e^{-x}$ for all $x \in \R_{\ge 0}$, we compute
\begin{equation*}%\label{eq:2}
\begin{split}
  \frac{1}{6}\frac{T_{\opt}(m,n)}{T_{p}(m,n)} 
  &\le \frac{\left(\frac{(1+\varepsilon)m}{n}\right)^m \left(1-\frac{(1+\varepsilon)m}{n}\right)^{n-m}}{\left(\frac{m}{n}\right)^m\left(1-\frac{m}{n}\right)^{n-m}} 
  = (1+\varepsilon)^m \left(1- \frac{m\varepsilon}{n-m}\right)^{n-m}  \\
 &\le \left(\frac{1+\varepsilon}{e^\varepsilon}\right)^m  \le \left(\frac{1+\varepsilon}{1+\varepsilon+\varepsilon^2/2}\right)^m  \\
 & = \left(1- \frac{\varepsilon^2/2}{1+\varepsilon+\varepsilon^2/2}\right)^m \le \left(1- \frac{\varepsilon^2}{5}\right)^m \le \exp\left(-\frac{m\varepsilon^2}{5}\right).
 \end{split}
\end{equation*}
Therefore, $T_p(m,n) \ge \frac{1}{6}\exp(\frac{m\varepsilon^2}{5}) T_{\opt}(m,n)$.\\
Similarly, for $q < (1-\varepsilon)\frac{m}{n}$, by using $e^{-x} \ge 1 - x +\frac{x^2}{2}-\frac{x^3}{6}$ valid for all $x\in \R$, we obtain
\begin{equation*}
\begin{split}
 \frac{1}{6}\frac{T_{\opt}(m,n)}{T_{q}(m,n)} &\le \frac{\left(\frac{(1-\varepsilon)m}{n}\right)^m \left(1-\frac{(1-\varepsilon)m}{n}\right)^{n-m}}{\left(\frac{m}{n}\right)^m\left(1-\frac{m}{n}\right)^{n-m}} \\ 
& = (1-\varepsilon)^m \left(1+ \frac{m\varepsilon}{n-m}\right)^{n-m} \le  \left(\frac{1-\varepsilon}{e^{-\varepsilon}}\right)^m  \\
& \le  \left(\frac{1-\varepsilon}{1-\varepsilon+\varepsilon^2/2-\varepsilon^3/6}\right)^m \le \left(1- \frac{\varepsilon^2/2-\varepsilon^3/6}{1-\varepsilon+\varepsilon^2/2-\varepsilon^3/6}\right)^m \\
 &\le \left(1-\frac{\varepsilon^2}{3}\right)^m \le \exp \left(-\frac{m\varepsilon^2}{3}\right).
\end{split}
\end{equation*}
Therefore, $T_q(m,n) \ge \frac{1}{6}\exp\left(\frac{m\varepsilon^2}{3}\right) T_{\opt}(m,n)$.
\end{proof}

%The proof lies on a relevant estimation of the different terms. Additionally, choosing $p$ larger than $2m/n$ (which means $\eps \ge 1$) incurs a runtime increase by a factor exponential in $m$.

\section{Design and Analysis of Heavy-tailed Mutation Operators}

In the previous section, we observed that an asymptotically optimal mutation rate for the $\jump_{m,n}$ function is $m/n$ rather than the general suggestion of $1/n$. However, due to the strong concentration of the number of bits that are flipped, we also observed that a small constant factor deviation from this optimal mutation rate incurs a significant increase in the runtime (exponential in $m$). From the view-point of algorithms design, this suggests that to get a good performance when optimizing multimodal functions, the algorithm designer needs to know beforehand which multi-bit flips will be needed to escape local optima. This is, clearly, an unrealistic assumption for any real-world optimization problem. To overcome this difficulty, we now design a mutation operator such that the number of bits flipped is not strongly concentrated, but instead follows a heavy-tailed distribution, more precisely, a power-law distribution. 

We prove that the \oea with this operator, which we shall call \foea, optimizes all $\jump_{m,n}$ functions with a runtime  larger than the optimal runtime $T_{\opt}(m,n)$ only by a small factor polynomially bounded in $m$, which is much better than the exponential (in $m$) performance loss incurred from missing the optimal static mutation rate by a few percent. We also show that such a small polynomial loss is unavoidable when aiming for an algorithm that shows a good performance on all jump functions. %\merk{can we say something on the case that we know the $m$ apart from constant factors?}
Finally, we show that similar performance gains from using a heavy-tailed mutation operator can also be observed with two combinatorial optimization problems, namely the problem of computing small vertex covers in complete bipartite graphs and the problem of computing large matchings in arbitrary graphs. 

\subsection{The Heavy-tailed Mutation Operator $\fmutb$}

The main reason why only a very carefully chosen mutation rate gave near-optimal results in Corollary~\ref{cor:optp} is the strong concentration behavior of the binomial distribution. If we flip bits independently with probability $m/n$, then with high probability the actual number of bits flipped is strongly concentrated around $m$. The probability that we flip  $(1-\varepsilon)m/n$ bits and less, or $(1+\varepsilon)m/n$ bits and more, at most $2 \exp(-\varepsilon^2 m / 3)$, that is, is exponentially small in $m$ (this follows directly from classic Chernoff bounds, e.g.,~\cite[Corollary~1.10 (a) and (c)]{Doerr11}). %my book chapter
 Hence to obtain a good performance on a wider set of jump functions (that is, with parameter $m$ varying at least by small constant factors), we cannot employ standard-bit mutation with a static mutation rate. 

To overcome the negative effect of strong concentration and at the same time be structurally close to the established way to performing mutation, we propose to use standard-bit mutation with a mutation rate that is chosen randomly in each iteration according to a power-law distribution with (negative) exponent $\beta$ greater than $1$. This keeps the property of standard-bit mutation with probability $1/n$ that with constant probability a single bit is flipped. This property is important to have a good performance in easy unimodal parts of the search space, and in particular, to easily approach the global optimum once one has entered its basin of attraction. At the same time, the heavy-tailed choice of the mutation rate ensures that with probability $\Theta(k^{-\beta})$, exactly $k$ bits are flipped. Hence this event, necessary to leave a local optimum with $(k-1)$ Hamming ball around it being part of its basin of attraction, is much more likely than when using the classic mutation operator, which flips $k$ bits only with probability $k^{-\Theta(k)}$. 

To keep the operator and its analysis simple, we only use mutation rates of type $\alpha / n$ with $\alpha \in [1..n/2]$. We show at the end of this section that no random choice of the mutation rate (including continuous ones) can give a performance on jump functions significantly better than the one stemming from our mutation operator, which justifies this restriction to integer values for $\alpha$. To ease reading, we shall always write $n/2$ even in cases where an integer is required, e.g., as boundary of the range of a sum. Of course, in all such cases $n/2$ is to be understood as $\lfloor n/2 \rfloor$. 

\textbf{The discrete power-law distribution $D_{n/2}^\beta$:} Let $\beta > 1$ be a constant. Then the discrete power-law distribution $D^\beta_{n/2}$ on $[1..n/2]$ is defined as follows. If a random variable $X$ follows the distribution $D^\beta_{n/2}$, then \[\Pr[X = \alpha] = (C_{n/2}^\beta)^{-1} \alpha^{-\beta}\] for all $\alpha \in [1..n/2]$, where the normalization constant is $C_{n/2}^\beta := \sum_{i=1}^{n/2} i^{-\beta}$. Note that $C_{n/2}^\beta$ is asymptotically equal to $\zeta(\beta)$, the Riemann zeta function $\zeta$ evaluated at $\beta$. We have \[\zeta(\beta) - \frac{\beta}{\beta-1}\left(\frac{n}{2}\right)^{-\beta+1} \le C_{n/2}^\beta \le \zeta(\beta)\]
for all $\beta>1$. As orientation, e.g., $\zeta(1.5) \approx 2.612$, $\zeta(2) \approx 1.645$, and $\zeta(3) = 1.202$ are some known values of the $\zeta$ function. 

\textbf{The heavy-tailed mutation operator $\fmutb$:} We define the mutation operator $\fmutb$ (with the $\texttt{f}$ again referring to the word  \emph{fast} usually employed when heavy-tailed distributions are used) as follows. When the parent individual is some bit-string $x \in \{0,1\}^n$, the mutation operator $\fmutb$ first chooses a random mutation rate $\alpha/n$ with $\alpha \in [1..n/2]$ chosen according to the power-law distribution $D_{n/2}^\beta$ and then creates an offspring by flipping each bit of the parent independently with probability $\alpha/n$ (that is, via standard-bit mutation with rate $\alpha/n$). The pseudocode for this operator is given in Algorithm~\ref{alg:heavy}.

\begin{algorithm2e}[h]%
Input: $x \in \{0,1\}^n$\\
Output: $y \in \{0,1\}^n$ obtained from applying standard-bit mutation to $x$ with mutation rate $\alpha/n$, where $\alpha$ is chosen randomly according to $D_{n/2}^\beta$\\
$y \assign x$\;
Choose $\alpha \in [1..n/2]$ randomly according to $D_{n/2}^\beta$\; 
\For{$j = 1$ \KwTo $n$}{
  \If{$\random([0,1]) \cdot n \le \alpha$}{
    $y_j \assign 1-y_j$;
    }
	}
\Return $y$  
\caption{The heavy-tailed mutation operator $\fmutb$.}
\label{alg:heavy}
\end{algorithm2e}

We collect some important properties of the heavy-tailed mutation operator. Again, we have to skip some of the proofs.

\begin{lemma}\label{lem:elementary}
  Let $n \in \N$ and $\beta > 1$. Let $x \in \{0,1\}^n$ and $y = \fmutb(x)$. 
  \begin{enumerate}
	  \item \label{it:elementary:1} The probability that $x$ and $y$ differ in exactly one bit, is $P^\beta_1 := \Pr[H(x,y)=1] = (C_{n/2}^\beta)^{-1}\Theta(1)$ with the constants implicit in the $\Theta(\cdot)$ independent of $n$ and $\beta$.
	  \item For any $k \in [2..n/2]$ with $k > \beta-1$, the probability that $x$ and $y$ differ in exactly $k$ bits, is $P^\beta_k := \Pr[H(x,y)=k] \ge (C_{n/2}^\beta)^{-1}\Omega(k^{-\beta})$ with the constants implicit in the $\Theta(\cdot)$ independent of $n$, $k$, and $\beta$. 
	  	  \label{it:elementary:2}
	  \item\label{it:worstprob} Let $z \in \{0,1\}^n$. If $\beta-1 < H(x,z) \le n/2$ or $H(x,z)=1$, then $\Pr[\fmutb(x)=z] = (C_{n/2}^\beta)^{-1} \Omega(H(x,z)^{-\beta}) \binom{n}{H(x,z)}^{-1}$. Without any assumption on $H(x,z)$, we have $\Pr[\fmutb(x)=z]  \ge (C_{n/2}^\beta)^{-1} \Omega( 2^{-n}n^{-\beta})$. In both cases, the implicit constants can be chosen independent of $z$, $\beta$, and $n$.
	  \item \label{it:expect_hamming} The expected number of bits that $x$ and $y$ differ in is \[
\expect{H(x,y)} =
\begin{cases}
\Theta(1) & \mbox{if } \beta > 2,\\
\Theta(\ln(n)) & \mbox{if } \beta = 2,\\
\Theta(n^{2-\beta}) & \mbox{if } 1 < \beta < 2,
\end{cases}
\]
where the implicit constants may depend on $\beta$.
  \end{enumerate}
\end{lemma}

\begin{proof}
\begin{enumerate}
	\item We have
	\begin{equation*}
		\begin{split}
		\Pr[H(x,y)=1]&= (C_{n/2}^\beta)^{-1}\sum_{i=1}^{n/2}\frac{1}{i^{\beta}}\binom{n}{1}\frac{i}{n}\left(1-\frac{i}{n}\right)^{n-1}\\
		&=(C_{n/2}^\beta)^{-1}\sum_{i=1}^{n/2}i^{1-\beta}\left(1-\frac{i}{n}\right)^{n-1}.
		\end{split}
	\end{equation*}
	By using $(1-\frac{1}{x})^x \leq e^{-1}$ valid for $x\neq 1$ and $e^x \ge x^2$ valid for $x \ge 0$, we obtain
	\begin{equation*}
	\begin{split}
	& \sum_{i=1}^{n/2}i^{1-\beta}\left(1-\frac{i}{n}\right)^{n-1} \leq \sum_{i=1}^{n/2}\frac{e^{-\frac{i(n-1)}{n}}}{i^{\beta-1}}\\
	& \leq \frac{n^2}{(n-1)^2}\sum_{i=1}^{n/2}\frac{1}{i^{\beta+1}} \leq 4\sum_{i=1}^{n/2}\frac{1}{i^2} < \frac{2\pi^2}{3}.
	\end{split}
	\end{equation*}
	Moreover, since $\left(1-\frac{1}{n}\right)^{n-1} \ge e^{-1}$ for every $n \ge 2$, we have $\sum_{i=1}^{n/2}\frac{1}{i^{\beta}}\binom{n}{1}\frac{i}{n}\left(1-\frac{i}{n}\right)^{n-1} \ge \left(1-\frac{1}{n}\right)^{n-1} \ge e^{-1}$. Hence, $\Pr[H(x,y)=1] = (C_{n/2}^\beta)^{-1}\Theta(1)$.

	\item We have
	\begin{equation*}
	\begin{split}
	& \Pr[H(x,y)=k] = (C_{n/2}^\beta)^{-1} \sum_{i=1}^{n/2} i^{-\beta}\binom{n}{k}\left(\frac{i}{n}\right)^{k}\left(1-\frac{i}{n}\right)^{n-k} \\
	& =\frac{1}{C_{n/2}^\beta n^{\beta-1}}\binom{n}{k}\left(\frac{1}{n}\sum_{i=1}^{n/2} \left(\frac{i}{n}\right)^{k-\beta}\left(1-\frac{i}{n}\right)^{n-k}\right).
	\end{split}
	\end{equation*}		
Since $\left(\frac{i}{n}\right)^{k-\beta}\left(1-\frac{i}{n}\right)^{n-k} \ge \left(1-\frac{i}{n}\right)^{k-\beta}\left(\frac{i}{n}\right)^{n-k}$ for any $i \le n/2$ and $k \le n/2$,
\begin{equation*}
\begin{split}
& \frac{1}{2n}\sum_{i=1}^{n} \left(\frac{i}{n}\right)^{k-\beta}\left(1-\frac{i}{n}\right)^{n-k}\le \frac{1}{n}\sum_{i=1}^{n/2} \left(\frac{i}{n}\right)^{k-\beta}\left(1-\frac{i}{n}\right)^{n-k}\\
& \le \frac{1}{n}\sum_{i=1}^{n} \left(\frac{i}{n}\right)^{k-\beta}\left(1-\frac{i}{n}\right)^{n-k}.
\end{split}
\end{equation*}

Exploiting the fact that $x \mapsto x^{k-\beta}(1-x)^{n-k}$ is unimodal in $[0,1]$, we approximate the previous expression by an integral as follows.
\begin{align*}
\frac{1}{n}\sum_{i=1}^{n} & \left(\frac{i}{n}\right)^{k - \beta} \left(1-\frac{i}{n}\right)^{n-k} \\
&= \int_{0}^{1}x^{k-\beta}(1-x)^{n-k}dx \pm \tfrac 1n \max\{x^{k-\beta}(1-x)^{n-k} \mid x \in [0,1]\} \\
&= \Theta\left(\int_{0}^{1}x^{k-\beta}(1-x)^{n-k}dx\right).
\end{align*}
%where the last estimate follows from \merk{add an argument}.

%
%To be more precise, we can prove:
%\[\frac{1}{3}\left(\int_{0}^{1}x^{m-\beta}(1-x)^{n-m}dx\right) \le \frac{1}{n}\sum_{i=1}^{n}\left(\frac{i}{n}\right)^{m - \beta} \left(1-\frac{i}{n}\right)^{n-m} \le 3\left(\int_{0}^{1}x^{m-\beta}(1-x)^{n-m}dx\right)\]

%\merk{old version}
%The integral is known as Beta function, we have a well-known formula as follows
%\begin{equation*}
%\begin{split}
%& B(k-\beta + 1, n-k+1) = \frac{\Gamma(k - \beta + 1)\Gamma(n - k + 1)}{\Gamma(n + 2 - \beta)} \\
%& =\frac{(n-k)!(\prod_{i=\lceil\beta\rceil + 2}^{k}(i - \beta))\Gamma(\lceil\beta\rceil + 1-\beta) }{(\prod_{i=\lceil\beta\rceil + 2}^{n + 1}(i - \beta))\Gamma(\lceil\beta\rceil + 1 - \beta)} =\frac{(n-k)!}{\prod_{i=k+1}^{n+1}(i - \beta)}.
%\end{split}
%\end{equation*}
%
%\merk{new version}
Since $k > \beta-1$, this integral is the Beta function $B$ evaluated at $(k-\beta + 1, n-k+1)$. Using the well-known relationship to the Gamma function, we compute
\begin{equation*}
\begin{split}
& B(k-\beta + 1, n-k+1) = \frac{\Gamma(k - \beta + 1)\Gamma(n - k + 1)}{\Gamma(n + 2 - \beta)} \\
& =\frac{(n-k)!\Gamma(k - \beta + 1)}{(\prod_{i=k+1}^{n + 1}(i - \beta))\Gamma(k - \beta + 1)} =\frac{(n-k)!}{\prod_{i=k+1}^{n+1}(i - \beta)}.
\end{split}
\end{equation*}

Let  $\varepsilon = \beta - 1$ and  $A = n^{\varepsilon}\prod_{i=k}^{n}(i - \varepsilon)$. We estimate
\begin{equation*}
\begin{split}
& \ln(A) = \varepsilon\ln(n) + \sum_{i = k}^{n}\ln(i - \varepsilon) =\varepsilon\ln(n) + \sum_{i = k}^{n}\left(\ln\left(1 - \frac{\varepsilon}{i}\right) + \ln(i)\right) \\
& \le \varepsilon\ln(n) + \sum_{i = k}^{n}\left(\frac{-\varepsilon}{i} + \ln(i)\right) \\
& = \varepsilon\ln(n) -\varepsilon(\ln(n) - \ln(k) + O(k^{-1})) + \sum_{i = k}^{n}\ln(i) \\
& =\varepsilon\ln(k) + \sum_{i = k}^{n}\ln(i) + O\left(\frac{\varepsilon}{k}\right)  \le \varepsilon\ln(k) + \sum_{i = k}^{n}\ln(i) + O(1),
\end{split}
\end{equation*}
Hence,
\[\Pr[H(x,y)=k] \ge (C_{n/2}^\beta)^{-1} \binom{n}{k}\frac{(n-k)!k!}{k^\beta n! O(1)} = (C_{n/2}^\beta)^{-1} \Omega(k^{-\beta})\]
with all asymptotic notation only hiding absolute constants independent of $n$, $k$, and $\beta$.

	\item The first part simply follows from~\ref{it:elementary:1} and~\ref{it:elementary:2} by noting (a) that there are  $\binom{n}{k}$ search points $z$ such that $H(x,z)=k$ and (b) that $y$ is uniformly distributed on these when we condition on $H(x,y) = H(x,z)$. For arbitrary~$z$, we use the fact than when the mutation rate is $1/2$, then the offspring is uniformly distributed in $\{0,1\}^n$. Hence $\Pr[\fmutb(x)=z \mid \alpha = n/2] = 2^{-n}$ and $\Pr[\fmutb(x)=z] \ge (C_{n/2}^\beta)^{-1} \Omega((n/2)^{-\beta}) 2^{-n}$. 
		
	\item Allowing all implicit constants to depend on $\beta$ in this paragraph, we calculate 
	\begin{align*}
	\expect{H(x,y)} &= \sum_{k=1}^{n/2} \Pr[\alpha = k] \, \expect{H(x,y) \mid \alpha = k}\\
	&= \Theta\left(\sum_{k=1}^{n/2} k^{-\beta} \cdot k\right) 
	= \Theta\left(\sum_{k=1}^{n/2} k^{1-\beta}\right).
	\end{align*}
	For $\beta <2$, we approximate the sum by the integral and obtain $\sum_{k=1}^{n/2} k^{1-\beta} = \Theta(\int_0^{n/2} x^{1-\beta} dx) = \Theta(\frac{x^{2-\beta}}{2-\beta}\big|_0^{n/2}) = \Theta(n^{2-\beta})$. For $\beta > 2$, the series $\sum k^{1-\beta}$ converges, and for $\beta = 2$, the sum over the first $n/2$ terms of the Harmonic series is $\Theta(\ln(n))$. Hence,
\[
\expect{H(x,y)} =
\begin{cases}
\Theta(1) & \mbox{if } \beta > 2,\\
\Theta(\ln(n)) & \mbox{if } \beta = 2,\\
\Theta(n^{2-\beta}) & \mbox{if } 1 < \beta < 2.
\end{cases}
\]
\end{enumerate}
\end{proof}

\subsection{Evolutionary Algorithms Using the Heavy-tailed Mutation Operator $\fmutb$}

Since we decided to call algorithms using the heavy-tailed mutation operator $\fmutb$ \emph{fast evolutionary algorithms}, we denote the \oea using the operator $\fmutb$ from now by \foea. We do likewise for any other $(\mu+\lambda)$~EA, which we call $(\mu+\lambda)$~FEA$_\beta$ when it employs the mutation operator $\fmutb$.

In this first section analyzing the performance of fast EAs, we show that many runtime analyses remain valid for the corresponding fast EA (apart from changes in the leading constant, which in many classic results is not made explicit anyway). An elementary observation is that fast EAs use the mutation rate $1/n$ with constant probability (Lemma~\ref{lem:elementary}~\ref{it:elementary:1}). Consequently, all previous runtime analysis which are robust to interleaving with other mutation steps remain valid for the FEAs (apart from constant factor changes of the runtime). These are, in particular, all analyses of elitist EAs based on the fitness level method~\cite{Wegener01} and on drift arguments using the fitness as potential function. We list some such results in the following theorem. The reference points to the original work for the non-fast EA.

\begin{theorem}
Let $\beta > 1$. 
\begin{enumerate}
	\item The expected runtime of the \foea on the \onemax and \leadingones test functions are $O(n \log n)$~\cite{Muhlenbein92} and $O(n^2)$~\cite{Rudolph97}, respectively.
	\item The expected runtime of the \foea on the minimum spanning tree problem is $O(m^2 \log(n w_{\max}))$~\cite{NeumannW07}.
	\item For all $\lambda \le n^{1-\eps}$, the expected runtime of the $(1+\lambda)$~FEA$_\beta$ is $O(\frac{n \lambda \log\log(\lambda)}{\log(\lambda)} + n \log n)$ on \onemax, it is $O(n^2 / \lambda)$ on \leadingones~\cite{JansenJW05,DoerrK15}, and it is $O(m^2 \log(n w_{\max}) / \lambda)$ for the minimum spanning tree problem~\cite{NeumannW07}.
  \item For all $\mu \le n^{O(1)}$, the expected runtime of the $(\mu+1)$~FEA$_\beta$ is $O(\mu n + n \log(n))$ for \onemax and $O(\mu n \log(n) + n^2)$ for \leadingones~\cite{Witt06}.
\end{enumerate}
\end{theorem}
%\merk{MST: Improved constant due to higher rate of 2-bit flips for certain betas?}

For the classic \oea with mutation rate $1/n$, it is known that it finds the optimum of any pseudo-Boolean fitness function in an expected number of at most $n^n$ iterations. This bound is tight in the sense that there are concrete fitness functions for which an expected optimization time of $\Omega(n^n)$ could be proven. These are classic results from~\cite[Theorem~6 to~8]{DrosteJW02}. 

Moreover, also problems that generally are perceived as easy may have instances where the classic \oea needs $n^{\Theta(n)}$ time to find the optimum. This was demonstrated for the minimum makespan scheduling problem~\cite{Witt05}, see also~\cite[Theorem~7.5 and Lemma 7.8]{Neumann2012BioinspiredCI}. While in average ($n$ jobs having random lengths in $[0,1]$) the classic \oea approximates the optimum up to an additive error of $1$ in time $O(n^2)$, there are instances of $n$ jobs with processing times in $[0,1]$ such that the \oea needs time $n^{\Omega(n)}$ to only find a solution that is better than $\frac 43$ times the optimum. 

The following result shows that fast EAs only have an exponential worst-case runtime as opposed to the super-exponential times  just discussed.

\begin{theorem}\label{thm:worstcase}
  Let $n \in \N$ and $\beta > 1$. Consider any fast EA creating at least a constant ratio of its offspring via the $\fmutb$ mutation operator. Then its expected optimization time on any fitness function $f : \{0,1\}^n \to \R$ is at most $O(C^\beta_{n/2} 2^n n^{\beta})$.
\end{theorem}

\begin{proof}
The claim follows immediately from Lemma~\ref{lem:elementary}~\ref{it:worstprob}: Since each offspring generated using $\fmutb$, regardless of the current state of the algorithm, with probability at least $\Omega(2^{-n} n^{-\beta})$ is an optimal solution, the expected optimization time is at most the reciprocal of this number.
\end{proof}

\subsection{Runtime Analysis for the \foea Optimizing Jump Functions}

We now proceed with analyzing the performance of the \foea on jump functions. We show that the \foea for all $m \in [2..n/2]$ with $m > \beta-1$ has an expected optimization time of $O(C^\beta_{n/2} m^{\beta-0.5} T_{\opt}(m,n))$ for the function $\jump_{m,n}$. By a mild abuse of notation, we denote by $T_\beta(m,n)$ the expected optimization time of the \foea on the $\jump_{m,n}$ function.

\begin{theorem}
\label{th:runtime}
  Let $n \in \N$ and $\beta > 1$. For all $m \in [2..n/2]$ with $m > \beta-1$, the expected optimization time $T_\beta(m,n)$ of the \oea with mutation operator $\fmutb$ is \[T_\beta(m,n) = O(C_{n/2}^\beta m^{\beta-0.5} T_{\opt}(m,n)),\] where the constants implicit in the big-Oh notation can be chosen independent from $\beta$, $m$ and $n$.
\end{theorem}

We do not discuss the case $m \le \beta-1$. For $\beta<3$, this case does not exist, and we do not have any indication that larger values of $\beta$ are useful.

\begin{proof}
We use the same notation as in the proof of Theorem~\ref{thm:optp} and in the statement of Lemma~\ref{lem:elementary}. For $i\in [0..n-1]$, let
\[
s_i \coloneqq
\begin{cases}
\frac{n - i}{n} P^\beta_1 & \mbox{if } 0 \leq i \leq n - m - 1,\\
\binom{n}{m}^{-1} P^\beta_m & \mbox{if } i = n - m,\\
\frac{i}{n} P^\beta_1 & \mbox{otherwise.}
\end{cases}
\]
Then $s_i$ is a lower bound for the probability that an iteration starting with a search point $x\in A_i$ ends with a search point of higher fitness. By the fitness level theorem, we obtain an upper bound for $T_\beta(m,n) $ by computing
\begin{align}
 T_\beta(m,n) &\leq \sum_{i = 0}^{n - 1}\frac{1}{s_i} \nonumber\\ 
& = \sum_{i = 0}^{n - m -1}\frac{n}{(n - i) P^\beta_1}  + \binom{n}{m} (P^\beta_m)^{-1} +\sum_{i = n - m + 1}^{n - 1} \frac{n}{i P^\beta_1}.\label{eq:sum}
\end{align}

To estimate the second term, we use the Stirling approximation $\sqrt{2\pi} n^{n+0.5} e^{-n} \le n! \le e n^{n+0.5} e^{-n}$ valid for all $n \in \N$. Using Lemma~\ref{lem:elementary} and noting that $n / (n-m) \le 2$, we compute
\begin{align*}
  \binom{n}{m} (P^\beta_m)^{-1} &= C^\beta_{n/2} O(m^\beta) \frac{n!}{m! (n-m)!} \\
  &= O(1) C^\beta_{n/2} m^{\beta-0.5} \frac{n^n}{m^m (n-m)^{n-m}} \\
  &= O(1) C^\beta_{n/2} m^{\beta-0.5} T_{\opt}(m,n). 
\end{align*}
Since this term is at least $\Omega(C^\beta_{n/2} n^2)$, whereas the sum of the first and third term in~\eqref{eq:sum} is at most $O(C^\beta_{n/2} n \log n)$, we have $ T_\beta(m,n) = O(C^\beta_{n/2} m^{\beta-0.5} T_{\opt}(m,n))$ as claimed.

%Moreover, denote by $x_{n-m} $ a search point in level $A_{n-m}$ and let $x_{\opt} = (1,...,1)$, we have 
%\[ E_{n-m} = \frac{1}{\Pr[\fmutb(x_{n-m})=x_{\opt}]}.\]
%By Lemma  \ref{lem:elementary}~\ref{it:elementary:2}, $\Pr[\fmutb(x_{n-m})=x_{\opt}] = \left(C_{n/2}^\beta \binom{n}{m}\right)^{-1}\Omega\left(m^{-\beta}\right)$, hence,
%\[E_{n-m}= C_{n/2}^\beta \binom{n}{m}O\left(m^{\beta}\right)= O\left(1\right)\frac{C_{n/2}^\beta m^{\beta}n!}{m!(n-m)!},\]
%where the constant implicit in the big-Oh notation is independent of $n$ and $\beta$. By applying Stirling's approximation $n!\approx \sqrt{2\pi n}\left(\frac{n}{e}\right)^n$ and $m!\approx \sqrt{2\pi m}\left(\frac{m}{e}\right)^m$, $(n-m)!\approx \sqrt{2\pi (n-m)}\left(\frac{n}{e}\right)^{n-m}$, notice that since $m \le n/2$, $\sqrt{\frac{n}{n-m}} \le \sqrt{2}$, we obtain
%\[E_{n-m} = \frac{O(1)C_{n/2}^\beta m^{\beta - 0.5}n^n}{m^m (n-m)^{n-m}}.\]
%Moreover, similar to the argument used in Lemma \ref{lm:optp},
%\[\frac{2\ln(n/m)}{\frac{1}{n}(1-\frac{1}{n})^{n-1}\Pr[X = 1]} \le \frac{2C_{n/2}^\beta n^n}{m^m (n-m)^{n-m}}\]
%On the other hand, $T_{\opt}(m,n) = \frac{O(1)n^n}{m^m (n-m)^{n-m}}$. 
%So, we can conclude that $ T_\beta(m,n) = O\left(C_{n/2}^\beta m^{\beta - 0.5}T_{\opt}(m,n)\right)$, where the constants implicit in the big-Oh notation can be chosen independent from $\beta$, $m$ and $n$.
\end{proof}

We remark that this runtime analysis is tight, but given that we prove a very similar lower bound in the subsequent section, we omit a proof for this claim.

\subsection{A Lower Bound for a Uniformly Good Performance on all Jump Function}

The runtime analysis of the previous subsection showed that the \oea with the heavy-tailed mutation operator optimizes any $\jump_{m,n}$ function in an expected time that is only a small polynomial (in $m$) factor larger than the runtime stemming from the (for this $m$) optimal mutation rate. When aiming at algorithms that are not custom-tailored for a particular value of $m$, this is a great improvement over any fixed mutation rate, which gives a runtime slower than $T_{\opt}(m,n)$ by a factor exponential in $m$ for many values of $m$, see Corollary~\ref{cor:optp}. Still, the question remains if this relatively small polynomial factor increase is necessary. In this section, we answer this affirmatively. While taking $\beta = 1 + \varepsilon$ can reduce the loss factor to $\Theta(m^{0.5+\varepsilon})$ for any $\varepsilon>0$, no randomized choice of the mutation probability can uniformly obtain a loss factor of $\sqrt m$ (or lower). To this aim, let us extend the definition of $T_\beta(m,n)$ and denote by $T_D(m,n)$ the expected number of iterations it takes the \oea to find the optimum of $\jump_{m,n}$ when in each iteration the mutation rate is chosen randomly according to the distribution $D$. 

\begin{theorem}
\label{th:lower_bound_unif_performance}
Let $c > 0$ and let $n$ be sufficiently large. Then for every distribution $D$ on $[0,1/2]$, there exists an $m \in [2..n/2]$ such that \[T_{D}(m, n)  \ge c\sqrt{m}T_{\opt}(m,n).\]
\end{theorem}

\begin{proof}
As in the proof of Theorem~\ref{thm:optp}, we first argue that also when using a random mutation rate $p$ distributed according to $D$, then the expected optimization time (essentially) is at least the time needed to jump from the second highest fitness level to the optimum. This part of the proof is very similar to Theorem~\ref{thm:optp}. For reasons of completeness, we still give it in the following.

We partition the search space into the nonempty level $A_i = \{x \mid  \lVert x\rVert_1 = i\}$, $0 \leq i \leq n$ and denote by $E_{i, D}$ the expected number of iterations it takes the algorithm to find the optimum when the starting point is at level $A_{i}$. Let $\alpha_{i,j}$ be the probability that one iteration of the algorithm (consisting of mutation and selection) starting at level $A_i$ ends in level $A_j$. We prove by induction that for $0 \leq i \leq n - m $, we have $E_{i,D} \geq E_{n - m, D}$. This trivially holds for $i = n - m$. Suppose that it holds for all $i$ such that $k + 1 \leq i \leq n - m$. We prove that it also holds for $i = k$. We have 
\[E_{k, D} =1 + \sum_{j = k}^{n - m} \alpha_{k,j}E_{j, D}.\]
By induction hypothesis,
\[(1 - \alpha_{k,k})E_{k, D} \geq 1 + \left(\sum_{j = k + 1}^{n - m} \alpha_{k,j}\right) E_{n - m, D},\]
and thus,
\[E_{k, D} \geq \frac{1}{1 - \alpha_{k,k}} + \frac{1 - \alpha_{k,k} - \alpha_{k, n}}{1 - \alpha_{k,k}} E_{n - m, D}.\]

Let $p$ be distributed according to $D$. Then $E_{n-m, D} =  \expect{(p^m(1 - p)^{n - m})^{-1}}$ and $\alpha_{k, n} = \expect{p^{n - k}(1 - p)^k}$. Therefore, since $p \le \frac{1}{2}$ and $k  \le n - m$, we have
$$p^{n - m - k}(1 - p)^{k - (n - m)} = \left( \frac{p}{1 - p}\right)^{n - m - k} \le 1$$
and thus 
$$p^{n-k}(1 - p)^k \le p^m(1 - p)^{n - m}.$$
Therefore, $ \expect{p^{n - k}(1 - p)^k} \le \expect{p^m(1 - p)^{n - m}} $, that is, $E_{n-m, D}\alpha_{k, n}  \leq 1$, and consequently $E_{k, D} \geq E_{n-m, D}$. By induction, we conclude that $E_{i, D} \geq E_{n - m, D}$ for all $0 \leq i \leq n - m $. 

Let $x$ denote the random initial search point. Then the above estimate yields
\begin{equation*}
    \begin{split}
        T_{D}(m,n)  & \geq \sum_{i = 0}^{n - m} \Pr[x \in A_i] E_{i, D} \\ & \geq \left(1 - \binom{n}{\le m-1} 2^{-n}\right) E_{n - m, D}\geq \frac{1}{2\;\expect{p^m(1 - p)^{n - m}}}.
    \end{split}
\end{equation*}
Assume that $D$ is such that for all $m \in [2..n/2]$, we have $T_{D}(m,n) < c\sqrt{m} T_{\opt}(m,n)$. Then
\[2 \; \expect{p^m (1 - p)^{n - m}} > \frac{1}{3c}\frac{1}{\sqrt{m}}\left(\frac{m}{n}\right)^m\left(1 - \frac{m}{n}\right)^{n - m}\]
holds for all $m \in [2..n/2]$. Multiplying this inequality by $\binom{n}{m}$ then summing over all $m \in [2..n/2]$, we obtain
\[2 \; \expect{\sum_{m = 2}^{n/2} p^m (1 - p)^{n - m} \binom{n}{m}} > \frac{1}{3c}\sum_{m = 2}^{ n/2 }\frac{1}{\sqrt{m}}\left(\frac{m}{n}\right)^m\left(1 - \frac{m}{n}\right)^{n - m}\binom{n}{m}.\]
The left-hand side is less than $ 2 \; \expect{\sum_{m = 0}^{n} p^m (1 - p)^{n - m} \binom{n}{m}} = 2 \; \expect{(p + (1 - p))^n} = 2$. Hence
\[2> 
\frac{1}{2c}\sum_{m = 2}^{ n/2 }\frac{1}{\sqrt{m}}\left(\frac{m}{n}\right)^m\left(1 - \frac{m}{n}\right)^{n - m}\binom{n}{m}. \]
With Stirling's approximation, we obtain
\begin{align*}
2 &> \frac{1}{2c}\sum_{m = 2}^{ {n/2} }\frac{1}{\sqrt{m}}\frac{m^m(n-m)^{n-m}}{n^n}\frac{\sqrt{2\pi n}n^n}{e \sqrt{m}m^m e \sqrt{(n - m)}(n - m)^{n-m}} \\
&=\Omega\bigg(\sum_{m = 2}^{ {n/2} }\frac{1}{m}\bigg) = \Omega(\log(n)),
\end{align*}
a contradiction.
\end{proof}

%Once more, we have to skip the proof. Our argument consisted in assuming that a distribution $D$ of the mutation rate exists such that for each $m$ in $[2..n/2]$, $T_{D}(m, n)  \le c\sqrt{m}T_{\opt}(m,n) $. Using a similar induction argument as for lemma \ref{lm:optp}, this would mean that for each $m$, the expected runtime before performing the final $m$-sized jump is also lower than $c\sqrt{m}T_{\opt}(m,n) $. However, this runtime is $\expect{p^m(1 - p)^{n - m}}^{-1}$, and for each fixed $p \in [0, 1]$ we have the binomial formula
%$$ 1 = (p + 1 - p)^n = \sum_{m=0}^n \binom{n}{m} p^m (1 - p)^{n - m}.$$
%Thus,
%$$ 1 = \sum_{m=0}^n \binom{n}{m} \expect{p^m (1 - p)^{n - m}} \ge \sum_{m=2}^{n/2} \binom{n}{m} \left(c\sqrt{m} T_{\opt}(m,n)\right)^{-1}$$
%and with some estimation work we get to something of the form
%$$ 1 \ge C \log(n),$$
%which is not true when $n$ is large enough.

%\section{Other Test Functions}

\subsection{Combinatorial Optimization Problems}

In this subsection, we show how our heavy-tailed mutation operator improves two existing runtime results for combinatorial optimization problems. Since the main argument in both cases is that some multi-bit flips used in the previous analyses now occur with much higher rate, we do not give (that is, repeat) the full proofs, but only sketch the main arguments. 

\subsubsection{Maximum Matching Problem} 

Let $\eps > 0$ be a constant. In~\cite{GielW03, GielW04}, see also~\cite[Section 6.2]{Neumann2012BioinspiredCI}, it is proven that the standard \oea in any undirected graph having $n$ edges finds a near-maximal matching of size at least $OPT / (1+\eps)$ in expected time  $O(n^{m+1})$, where $m = 2 \lceil\frac{1}{\epsilon} \rceil - 1$. We now show that the \foea improves this bound to $O(C_{n/2}^\beta(e/m)^m m^{\beta-0.5} n^{m+1})$, that is, roughly by a factor of $m^{\Theta(m)}$.

Let $G = (V, E)$ be an undirected graph. Let $n := |E|$. The \emph{maximum matching problem} consists in finding a largest subset $M^*$ of $E$ such no two edges in $M^*$ have a vertex in common. This problem can be solved via EAs by taking $S = \left\{0,1\right\}^n$ as search space with each bit encoding whether a given edge is part of the solution or not. For $s \in S$, let $p_{s}(v) = \max(0, d_s(v) - 1)$ for each vertex $v \in V$, where  $d_s(v)$ is the number of edges in $s$ that are incident with~$v$. Then $p(s) = \sum_{v \in V} p_s(v)$ is a penalty term measuring how far our solution deviates from being a matching. As fitness function we use $f(s) = (-p(s), \lVert s \rVert_1)$, which is to be maximized with respect to the lexicographic order. It is easy to see that both the \oea and the \foea in time $O(n \log n)$ reach a search point that is a matching. The key observation in~\cite{GielW03,GielW04} is that if $M$ is some suboptimal matching, then there exists an augmenting path with respect to $M$ whose length is bounded from above by $L := 2\lfloor|M|/(|M^{*}| - |M|)\rfloor + 1$. Consequently, by flipping all bits corresponding to edges on this path, we can increase the size of the matching by one. Hence after at most $n$ times flipping the bits of a path of length $m$, we have a matching of cardinality $|M^*|/(1+\eps)$. This takes time $O(n \cdot n^{m}) = O(n^{m+1})$ for the \oea, giving the result of~\cite{GielW03,GielW04}. For the \foea, by Lemma~\ref{lem:elementary}~\ref{it:worstprob}, this time is \[O\bigg(n C_{n/2}^\beta m^\beta \binom{n}{m}\bigg) = O\bigg(C_{n/2}^\beta \bigg(\frac{(1+o(1))e}{m}\bigg)^m m^{\beta-0.5} n^{m+1}\bigg).\]

%We first state the following lemma (its proof can be found in \cite{Neumann2012BioinspiredCI}):
%\begin{lemma}
%Let $G = (V, E)$ be a graph, $M$ a non-maximum matching, and $M^{*}$ a maximum matching. Then there exists an augmenting path with respect to $M$ whose length is bounded from above by $L := 2\lfloor|M|/(|M^{*}| - |M|)\rfloor + 1$.
%\end{lemma}
%We denote by $p \sim D_{n/2}^\beta \in [0, 1/2]$ the mutation rate given by the operator. The optimization can be done in two phases. In the first phase, the expected time to go from an initial search point to a matching at most is \[\sum_{i = 1}^{n}\frac{1}{i\expect{ p (1 - p)^{n - 1}}} = O(n\ln(n)).\] Afterwards, let $M$ be the current matching, and let $M^{*}$ be an arbitrary maximum matching. The search is successful if $|M^{*}| \le (1+\epsilon)|M|$. Otherwise, there exists an augmenting path for $M$ whose length is bounded from above by $L := 2\lfloor|M|/(|M^{*}| - |M|)\rfloor + 1 \le m:= \lceil\frac{2}{\epsilon} \rceil - 1$. The expected waiting time to flip exactly the edges of an augmenting path of length $l \le L \le m$ is at most
%\begin{equation*}
%\begin{split}
%&\frac{1}{\expect{p^m (1 - p)^{n - \merk{m ?}}}}  = \frac{O(1)C_{n/2}^\beta m^{\beta - 0.5}}{\left(\frac{m}{n}\right)^m\left(1-\frac{m}{n}\right)^{n-m}} \\
%& = O(C_{n/2}^\beta m^{\beta - 0.5 - m}e^mn^m).
%\end{split}
%\end{equation*}
%In order to solve the problem we only need to wait this event to happen $|M^{*}| \le n$ times. This gives us the result.

\subsubsection{Vertex Cover Problem}

Friedrich, Hebbinghaus, Neumann, He and Witt \cite{FriedrichHHNW10} (see also \cite[Chapter 12]{Neumann2012BioinspiredCI}) analyze how evolutionary algorithms compute minimum vertex covers. They observe that already on complete bipartite graphs with partition classes of size $m \le n/2$ and $n-m$, the standard \oea has an expected optimization time of $\Omega(mn^{m-1} + n \log n)$ (to obtain this precise bound, an inspection of their proofs is necessary though). We now show that for $m \le n/3$ the \foea find the global optimum in only $O(C^b_{n/2} n^\beta 2^{m} )$ iterations. For $m \ge n/3$, our general bound of $O(C^\beta_{n/2} n^\beta 2^n)$ gives again a significant improvement over the classic \oea. For reasons of readability, we omit in the following the factor of $C^\beta_{n/2}$. Hence the statements may suppress a dependence of the constants on $\beta$, however, in all cases this would be only the factor $C^\beta_{n/2}$.

% look at the vertex cover problem on a complete bipartite graph $G = (V, E)$, \emph{i.e} such that one can write $ V = V_1 \bigsqcup V_2$, with each edge in $E$ having one end in $V_1$ and the other in $V_2$ ). If $\card{V} = n$ and $\card{V_1} = m < n/2$, they find that the classical \oea gets to the solution in expected time $O(n^{m})$. If we assume $m \le n/3$ then our heavy-tailed operator needs only $O( n^\beta 2^{m - \beta} )$, which is an improvement by a factor of $( n/2 )^{m - \beta}$.

Given a finite graph $G = (V, E)$, the \emph{vertex cover problem} consists of finding a subset $V' \subseteq V$ of minimal size such that each edge of the graph has at least one of its vertices in $V'$. If $\card{V} = n$, say $V = \lbrace v_1, \ldots, v_n \rbrace$, a candidate solution can be represented by a bit-string $x \in \{0, 1\}^{n} $ with $x_i = 1$ if and only if $v_i \in V'$. The fitness function (to be maximized) used in~\cite{FriedrichHHNW10} is $\left(- f(x), - g(x)\right)$, with $f(x)$ being the number of edges not covered by $x$, and $g(x) = \sum_{i = 1}^n x_i$ the number of vertices that $x$ encompasses. Again, fitness values are to be compared with respect to the lexicographical order. This means that solutions covering more edges are always accepted, up to the point where a solution is found that covers all edges; from then on we try to reduce the number of vertices while still covering all edges.

Both the standard \oea and the \foea within the first $O(n \log n)$ iterations construct a solution containing one of $V_1$ and $V_2$, that is, being a vertex cover. In another $O(n \log n)$ iterations, vertices are removed until the solution is one of $V_1$ or $V_2$. Both phases can be analyzed by regarding $1$-bit flips only. With probability $\Theta(m/n)$, this solution is the local optimum $V_2$. In this case, in a single step all $m$ bits representing the vertices in $V_1$ have to be flipped and at least $m$ other bits representing vertices of $V_2$ have to be flipped for the new solution to be accepted. For the \oea, this happens with probability at most $O(n^{-m})$. For the \foea, this happens with probability at least
$$\left( C_{n/2}^\beta \right)^{-1} \left(\frac{n}{2}\right)^{- \beta} \left(\frac{1}{2}\right)^m \left( \sum_{k = m}^{n - m} \binom{n - m}{k} \left(\frac{1}{2}\right)^{n - m} \right),$$
as can be seen from regarding only the iterations using a mutation rate of $1/2$.

We notice that $ \sum_{k = m}^{n - m} \binom{n - m}{k} /2^{n - m} $ is the probability that a binomial variable $Z \sim B(n - m, 1/2)$ takes any value of at least $m$. Since $m \le n/3$, we have $n - m  \ge 2n / 3$, so $(n - m)/2 \ge m$. So this probability is at least $\Pr(Z \ge \expect{Z}) \ge 1/2$. Consequently, the probability to switch a solution containing the other minimal cover, is at least $(C^\beta_{n/2})^{-1} (n/2)^{-\beta} 2^{-m-1}$. From this solution, simple $1$-bit flips in $O(n \log n)$ time lead to the optimal solution. Hence the expected optimization time is dominated by $C^\beta_{n/2} (n/2)^\beta 2^{m+1}$, the expected waiting time for switching into the basin of attraction of the global optimum. This proves our claim (with some room to spare).

\section{Experiments}

We ran an implementation of algorithm \ref{alg:heavy} against the jump function with $n$ varying between 20 and 150. For $m=8$, Figure \ref{graph_jump} shows the average number of iterations (on 1000 runs) of the algorithm before reaching the optimal value of $(1, \ldots, 1) \in \mathbb{R}^n$, with different values of the parameter $\beta$ and along with the ``classical'' \oea. We observe that small values of $\beta$ give better results, although no significant performance increase can be seen below $\beta = 1.5$ (which is why we depicted only cases with $\beta \ge 1.5$). The runtimes for $\beta=1.5$ uniformly are better than the one of the classic \oea by a factor of $2000$. %For $m=5$, not shown here, the improvement would still be by a factor of $5$ uniformly over all problem sizes regarded. 

\begin{figure}[h]
\center
\includegraphics[width=0.9\textwidth]{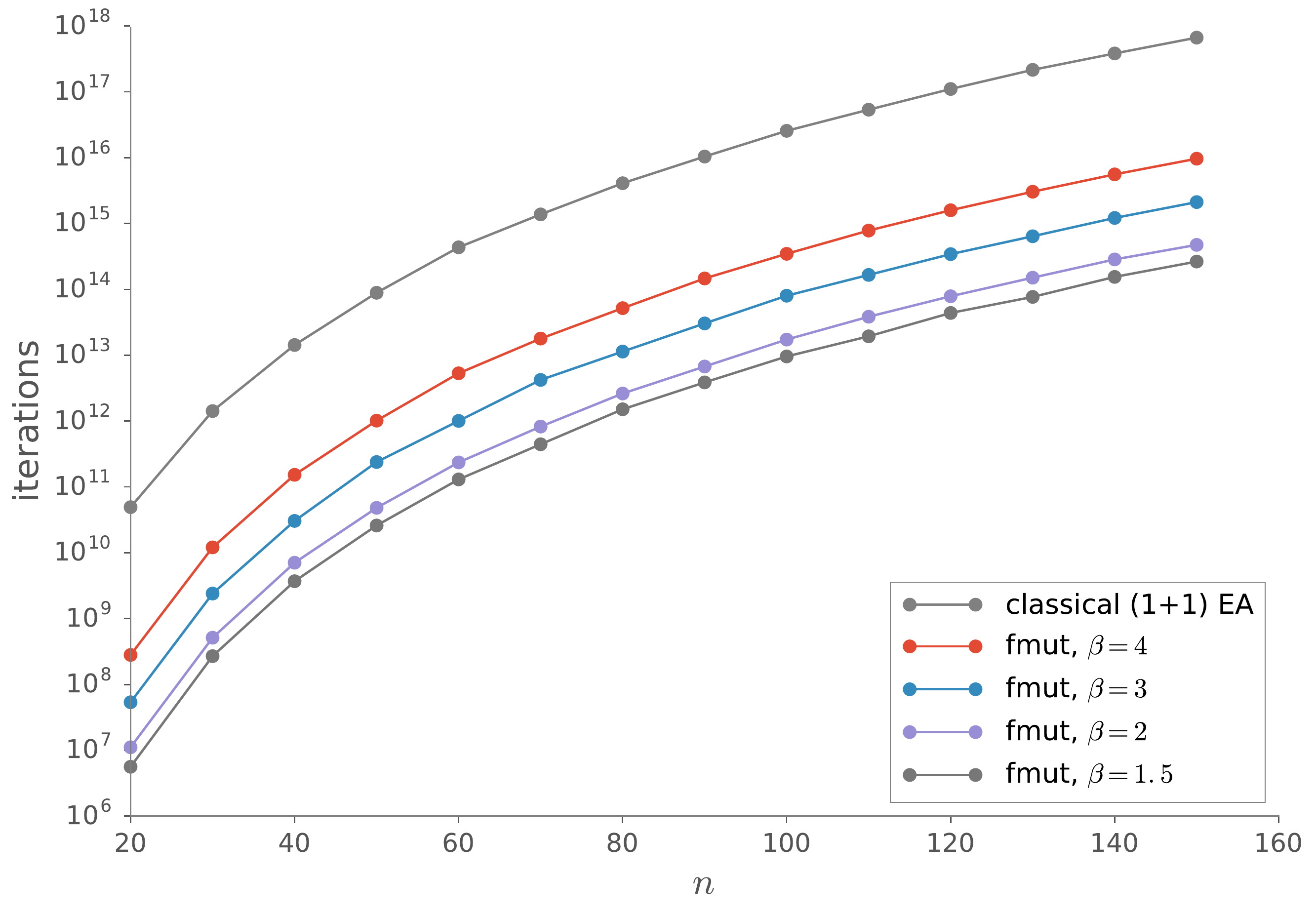}
\caption{Average number (over 1000 runs) of iterations the \oea took with different mutation operators to find the optimum of $\jump_{8,n}$.}
\label{graph_jump}
\end{figure}

\section{Conclusions}

In this work, we took a critical look at the performance of simple mutation-based evolutionary algorithms on multimodal fitness landscapes. Guided by the classic example of jump functions, we observed that mutation rates significantly above the usual recommendation of $1/n$ led to much better results. The proofs of our results suggest that when a multi-bit flip is necessary to leave a local optimum, then so much time is needed to find the right bits to be flipped that it is justified to use a mutation probability high enough that such numbers of bits are sufficiently often touched. The speed-up here greatly outnumbers the slow-down incurred in easy parts of the fitness landscape. 

Since we also observe that the optimal performance can only be obtained for a very small interval of mutation probabilities, we suggest to choose the mutation probability randomly according to a power-law distribution. We prove that this results in a ``one size fits all'' mutation operator, giving a nearly optimal performance on all jump functions. We observe that this mutation operator gives an asymptotically equal or better (including massively better) performance on many problems that were analyzed rigorously before. 

Let us remark that heavy-tailed mutation is not restricted to bit-string representations. For combinatorial problems for which a bit-string representation is inconvenient, e.g., permutations, one way of imitating standard-bit mutation is to sample a number $k$ according to a Poisson distribution with expectation $1$ and then perform $k$ elementary mutation steps, where an elementary mutation step is some simple modification of a search point, e.g., a random swap of two elements in the case of permutations~\cite{ScharnowTW04}. Obviously, to obtain a heavy-tailed mutation operator one just needs to replace the Poisson distribution with a heavy-tailed distribution, e.g., a power-law distribution as used in this work. We are optimistic that such approaches can lead to similar improvements as observed in this work, but we have not regarded this in detail.

Another important step towards understanding heavy-tailed mutation operators would be to gain experience on its performance on real applications. To make it easiest for other researchers to try our new methods, we have put the (simple) code we used in the repository~\cite{github}. %We plan to regularly update it, both with code and experience gained, so that is serves as an easy and up-to-date access point. 
%
%Overall, this work suggests that the previous recommendation to perform mutation via flipping each bit independently are overfitted to the easy unimodal fitness landscapes which were the first to be analyzed rigorously.

\subsection*{Acknowledgements}

This research was supported by Labex DigiCosme (project ANR11LABEX0045DIGICOSME) operated by ANR as part of the program ``Investissement d'Avenir'' Idex ParisSaclay (ANR11IDEX000302) as well as by a public grant as part of the Investissement d'avenir project, reference ANR-11-LABX-0056-LMH, LabEx LMH.

\bibliographystyle{alpha}
\bibliography{heavybib}

}%end sloppy
\end{document}